\documentclass[conference]{IEEEtran}

\usepackage{amssymb}
\usepackage{amsmath}
\usepackage{amsthm}
\usepackage{graphicx}
\usepackage{subfigure}

\newcommand\be{\[}
\newcommand\ee{\]}
\newcommand\sigmainv{\sigma^{-1}}
\newcommand\sigmastar{\sigma^*}
\newcommand\sigmastarinv{\sigma^{*-1}}
\newcommand\aaa{a}
\newcommand\astar{\aaa^{*}}

\newtheorem{theorem}{Theorem}

\begin{document}

\title{Improved Error Bounds Based on Worst Likely Assignments}

\author{\IEEEauthorblockN{Eric Bax}
\IEEEauthorblockA{
Email: baxhome@yahoo.com}}


\maketitle

\begin{abstract}
Error bounds based on worst likely assignments use permutation tests to validate classifiers. Worst likely assignments can produce effective bounds even for data sets with 100 or fewer training examples. This paper introduces a statistic for use in the permutation tests of worst likely assignments that improves error bounds, especially for accurate classifiers, which are typically the classifiers of interest.
\end{abstract}


\section{Introduction}
Permutation tests are used in statistics for hypothesis testing, especially in bioinformatics \cite{golland05,corcoran05,knijnenburg09,efron07,subramanian05}.  Permutation tests apply to a wide range of problems because they do not rely on assumptions about the form of the distributions that generate samples. Because they do not rely on asymptotic results, they are ideal for small-sample problems. As a result, permutation tests are often used in exact tests \cite{weerahandi95,fisher54}.   

Error bounds based on worst likely assignments \cite{baxcallejas,li12} incorporate permutation tests as part of a technique to produce error bounds for classifiers. Worst likely assignments can produce effective error bounds even for small data sets. In this paper, we show effective error bounds even for data sets with 100 or fewer training examples. 

Worst likely assignment error bounds apply to the transductive setting \cite{vapnik98}, where there is a set of training examples with known inputs and labels, and a set of working examples with known inputs and unknown labels. The goals are training---developing a classifier that performs well on the working examples---and validation---producing a bound on the classifier's error rate over the working examples. 

The worst likely assignment technique tests each potential assignment of class labels to the working examples, evaluating whether the assigned labels cause the working examples to ``blend in" among the training examples. If so, then the assignment is declared a likely assignment, making the classifier's error rate given the assignment a candidate for the error bound. If not, then the assignment is dismissed as being unlikely.

This paper presents a new type of statistic for use in the permutation test to determine whether an assignment is likely. The statistic is a \textit{scoring function} that evaluates whether assigned labels cause working examples to have rates of disagreement with neighboring examples' labels that are similar to those rates for training examples. The scoring function significantly improves error bounds for accurate classifiers. 

This paper is organized as follows. Section 2 defines terms and notation. Section 3 reviews error bounds based on worst likely assignments. Section 4 presents the new scoring function. Section 5 demonstrates how well the scoring function performs on some data sets. Section 6 closes with a discussion of challenges for future work.

\section{Concepts and Notation} 

This paper concerns validation of classifiers learned from examples. Each example \textit{Z} = (\textit{X}, \textit{Y}) includes an input \textit{X} and a class label \textit{Y} $\in$ \{0,1\}. Define \textit{complete sequence} $C$ to be a random variable:

\[C=Z_{1} , \ldots,Z_{t+w}\]

\noindent
with examples \(\textit{Z}_{1}=(\textit{X}_{1},\textit{Y}_{1})\), \dots , \(\textit{Z}_{t+w}=(\textit{X}_{t+w},\textit{Y}_{t+w})\) drawn i.i.d. from an unknown joint distribution $D$ of inputs and labels. We observe 

\be
(X_{1} ,Y_{1} ),...,(X_{t} ,Y_{t} ),X_{t+1} ,...,X_{t+w}
\ee

\noindent
that is, inputs and outputs of \textit{t} training examples and just the inputs of \textit{w} working examples.  A classifier \textit{g}, which is a mapping from the input space of \textit{X} to \{0,1\}, is developed using the observed data. Then classifier \textit{g} is used to predict the working example outputs \(\textit{Y}_{t+1}\), \dots , \(\textit{Y}_{t+w}\) associated with inputs \(\textit{X}_{t+1}\), \dots , \(\textit{X}_{t+w}\). 

For any sequence of examples

\be
 c=(x_{1} ,y_{1} ), \ldots, (x_{t+w} ,y_{t+w} ) 
\ee

\noindent
from the joint space of inputs and labels, define the \textit{error} to be

\be
E_{c} =\frac{1}{w} \sum _{i=t+1}^{t+w}I(g(x_{i} )\ne y_{i} ),
\ee

\noindent
where \textit{I} is the indicator function---one if the argument is true and zero otherwise. The goal is to produce a PAC (probably approximately correct) bound on \(\textit{E}_{C}\), the error on complete sequence \textit{C}.

\section{Worst Likely Assignment Error Bound}

Use $\sigma c$ to denote the sequence $c$ permuted by permutation $\sigma$ of $1, \ldots, t+w$. Use $\sigmainv$ to denote the inverse of permutation $\sigma$, such that $\sigma \sigmainv c = c$ for all $c$. Let $h()$ be a real-valued \textit{scoring function} on sequences of $t+w$ examples. For example, $h(c)$ could be the error rate over the last $w$ inputs for a classifier trained on the first $t$ examples in $c$. Let $Q$ be a set or multi-set of permutations of $\{1, \dots, t+w\}$. Define the \textit{ranking function} $r(x, S)$ to be the rank of value $x$ among the entries in set or multiset $S$, with random tie-breaking to determine the rankings among equal values. 

\begin{theorem}\label{thm1}

If $C$ is drawn according to $D^{t+w}$ and $\sigmastar$ is drawn uniformly at random from $Q$, then

\be
\forall k \in \{1, \ldots, |Q|\}:
\ee
\be
Pr[r(h(C), \{h(\sigma \sigmastarinv C) | \sigma \in Q\}) = k] = \frac{1}{|Q|}
\ee

\noindent
where the probability is over random draws of $C$ and $\sigmastar$. In other words, if $C$ is mapped to a random permutation in $Q$, then $C$ is equally likely to have each rank among permutations of $C$ relative to the mapped permutation.

\noindent
\end{theorem}
\begin{proof}
Since the elements of $C$ are i.i.d., for each sequence $c$, each permutation of $c$ is equally likely to be $C$. So the distribution of $\sigmastarinv C$ is the same as the distribution of $C$, and we can replace $\sigmastarinv C$ by $C$ in the theorem to get the logically equivalent statement:

\be
\forall k \in \{1, \ldots, |Q|\}: 
\ee
\be
Pr[r(h(\sigmastar C), \{h(\sigma C) | \sigma \in Q\}) = k] = \frac{1}{|Q|},
\ee

\noindent
where the probability is over random draws of $C$ and $\sigmastar$. To prove this equation, it is sufficient to show that 

\be
\forall c: \forall k \in \{1, \ldots, |Q|\}: 
\ee
\be
Pr[r(h(\sigmastar c), \{h(\sigma c) | \sigma \in Q\}) = k] = \frac{1}{|Q|},
\ee

\noindent
where the probability is over random draws of $\sigmastar$. Since $\sigmastar$ is selected uniformly at random from $Q$, $h(\sigmastar c)$ is equally likely to be each entry in 

\be
\{h(\sigma c) | \sigma \in Q\}.
\ee

\noindent
So $h(\sigmastar c)$ is equally likely to have each rank in $1, \ldots, |Q|$.
\noindent
\end{proof}

Let $\aaa = (\hat{y}_{t+1}, \ldots, \hat{y}_{t+w})$ denote an assignment to the (unknown) outputs of the working examples in sequence $C$, and let $C(\aaa)$ denote the sequence with the values in $\aaa$ assigned to the working example outputs.  Let $\astar = (y_{t+1}, \ldots, y_{t+w})$ be the actual outputs, so that $C(\astar) = C$. For a given \textit{bound failure probability} $\delta$, define a \textit{likely set} of assignments:

\be
L = \{\aaa \in \{0,1\}^w |
\ee
\be
r(h(C(\aaa)), \{h(\sigma \sigmastarinv C(\aaa)) | \sigma \in Q\}) \leq \lceil (1 - \delta) |Q| \rceil \}.
\ee

\noindent
In other words, the likely set contains the assignments $\aaa$ that rank in the bottom $1 - \delta$ in the test: assign $\aaa$ to the unknown outputs of the working set of $C$, permute $C(\aaa)$ by $\sigmastarinv$, take all permutations $\sigma$ in $Q$ of $\sigmastarinv C(\aaa)$, compute their scores  $h(\sigma \sigmastarinv C(\aaa))$, then check the rank of the score of $h(C(\aaa))$ among the scores. 

\begin{theorem}\label{thm2}
If $C$ is drawn according to $D^{t+w}$ and $\sigmastar$ is drawn uniformly at random from $Q$, then

\be
Pr[E_C \leq \max_{\aaa \in L} E_{C(\aaa)}] \geq 1 - \delta,
\ee

\noindent
where the probability is over random draws of $C$ and $\sigmastar$. 

\noindent
\end{theorem}
\begin{proof}
Since $C=C(\astar)$, 

\be
E_C = E_{C(\astar)}.
\ee

\noindent
So 

\be 
\astar \in L \implies E_C \leq \max_{\aaa \in L} E_{C(\aaa)}.
\ee

\noindent
Examine the definition of $L$. When $\aaa = \astar$, $C(\aaa) = C$. So 

\be
r(h(C), \{h(\sigma \sigmastarinv C | \sigma \in Q\}) \leq \lceil (1 - \delta) |Q| \rceil \implies \astar \in L.
\ee

\noindent
According to Theorem \ref{thm1}, 

\be
r(h(C), \{h(\sigma \sigmastarinv C | \sigma \in Q\})
\ee

\noindent
is equally likely to have each value in $1, \ldots, |Q|$. So the probability that the value is in $1, \ldots,  \lceil (1 - \delta) |Q| \rceil $ is at least $1 - \delta$.
\noindent
\end{proof}

Error bounds based on Theorem \ref{thm2} are called worst likely assignment error bounds because the bound is the highest error that is consistent with a likely assignment. Theorem \ref{thm2} is general. A specific error bound requires a scoring function $h()$ and a permutation set $Q$.  Effectiveness of the error bound and ease of computation influence the selection and design of scoring functions and permutation sets. Bax and Callejas \cite{baxcallejas} outline several scoring functions, including the error from developing a classifier based on the first $t$ examples and applying it to the last $w$ examples from the argument sequence of $t+w$ examples. They also also introduce two types of permutation sets: the set of all permutations and random subsets of permutations.  This paper introduces a new scoring function that improves error bounds.

\section{A Near Neighbor Scoring Function}
Consider a scoring function that counts differences between labels of examples in the last $w$ examples and their nearest neighbors in the first $t$ examples of the sequence $c$ being scored. Suppose that nearby neighbors among the training examples tend to have the same labels. Then the scores will typically be low when scoring permutations of $c$, because many of the examples in the last $w$ of the permuted sequence will be training examples, and in many cases their nearest neighbors in the first $t$ examples of the permuted sequence will be nearby neighbors among the training examples. These low scores constrain the set of likely assignments to those with high rates of agreement between labels assigned to the working examples and the labels of their nearest neighbors among the training examples. As a result, the scoring function typically produces strong error bounds. 

This section defines a class of scoring functions that count disagreements between examples in the last $w$ examples and their nearby neighbors in the first $t$ examples. The disagreements can be summed over multiple near neighbors to make the scores more robust. Also, the contributions to the sum from different neighbors can be weighted to emphasize nearer neighbors more. 

Let 

\be
n_{ijS}(c)
\ee

\noindent
be the label of the $i$th nearest neighbor to example $j$ in sequence $c = \{(x_1,y_1), \ldots, (x_{t+w},y_{t+w})\}$ among the examples indexed by set $S$, with random tie-breaking. Define a class of scoring functions

\begin{equation} \label{faks}
f_{\alpha k S}(c) =  \sum_{i=1}^{k} \sum_{j \in \{t+1, \ldots, t+w\}} \alpha^{i-1} I(n_{ijS} \neq y_j)
\end{equation}

\noindent
where the indicator function $I()$ is one if the argument is true and zero otherwise.

These scoring functions count disagreements between the last $w$ examples in $c$ and their nearest neighbors among a subset of examples in $c$. The parameter $\alpha$ specifies how much to weigh disagreements with nearby neighbors relative disagreements with more distant neighbors. The parameter $k$ expresses how many nearby neighbors to consider. Set $S$ constrains the examples for which disagreements can be counted. 

Setting $S = \{1, \ldots, t\}$ specifies that when the scoring function is applied to the sequence consisting of training examples followed by working examples with assigned labels, the score ignores disagreements between pairs of working examples, which both have assigned labels. So the scoring function does not favor assignments with incorrect, but agreeing, labels assigned to neighboring working examples. Instead, the score is based solely on disagreements between pairs of examples that have one working example with an assigned label and one training example with an actual label drawn by sampling. So it favors assignments that label working examples similarly to their neighboring training examples.

\section{Tests}
This section presents results from applying the near neighbor scoring function to produce error bounds for several data sets. For each data set, the results include a figure showing how adjusting the influence of more distant neighbors affects error bounds and a table showing more detailed statistics for the error bounds for a few parameter settings from the figure. For each data set, there are comparisons over a range of bound certainty values $\delta$. 

In these tests, NNSF refers to the near neighbor scoring function $f_{\alpha k S}$ defined in Equation \ref{faks}. The parameter settings are $\alpha \in \{0.0, \ldots, 0.9\}$ for the figures, $\alpha=0.5$ for the tables, and $k=t$ and $S=\{1, \ldots, t\}$ for both figures and tables. So in these tests, NNSF weighs each disagreement with a neighbor by $\alpha$ as much as disagreement with the next nearest neighbor, and, for each example in the last $w$ examples, disagreements with examples in the first $t$ examples contribute to the score. In the figures, as $\alpha$ increases from zero to one, more neighbors play significant roles in determining whether to accept or reject an assignment.

ESF (for error scoring function) refers to a baseline scoring function that uses the error from using the first $t$ examples as a 1-nearest neighbor classifier on the last $w$ examples. This scoring function, which was introduced by Bax and Callejas \cite{baxcallejas}, is equivalent to $f_{\alpha k S}$ with $k=1$ and $S=\{1, \ldots, t\}$. (We treat $0.0^0$ as one in Equation \ref{faks}, so ESF is also equivalent to NNSF with $\alpha = 0.0$.)

For each figure, each line shows results for a different value of the bound certainty parameter $\delta$. The value of $\alpha$ varies along each line. The plotted amounts are differences between error bound and actual error rate. Each plotted value is an average over 1000 trials. 

For each table, each row holds results for a different value of the bound certainty parameter $\delta$. The second column of each table shows errors from using training data as a 1-nearest neighbor classifier on working data. The subsequent columns show differences between the bounds on error and actual error for bounds using scoring functions NNSF and ESF.

Each cell shows a mean and standard deviation over 1000 trials. The cells in the ``Error'' column show mean and standard deviation of errors. The cells in subsequent columns show mean and standard deviation of difference between bound and error. For example, suppose the error has mean 0.3 and standard deviation 0.4, and a bounding method has mean 0.1 and standard deviation 0.0. This indicates that the error averages 0.3 over the 1000 trials, and the error varies quite a bit, but the bound is always exactly 0.1 greater than the actual error. 

Note that the standard deviations displayed in cells are standard deviations of the values over 1000 trials. They are not standard deviations of the estimates of the means of values over 1000 trials, that is, their large sizes do not indicate uncertainty about the accuracy of the means. Since there are 1000 trials, those standard deviations are about 1/33 of the ones shown, indicating that most differences in means of bounds produced by the different scoring functions are statistically significant for most of the tests.  

Each figure line and table row is based on the same 1000 trials, but different lines and rows are based on different sets of trials. For each trial, a size \textit{t+w} subset of examples is selected at random from a data set. A size-\textit{t} subset is selected at random to form the training set, and the remaining \textit{w} examples form the working set. The error is computed, and error bounds are computed using the scoring functions NNSF and ESF. The error is subtracted from each bound, and the differences are accumulated into the statistics shown in the figures and tables.

\begin{figure}[ht!]
     \begin{center}

        \subfigure[Iris Data]{
            \label{vary_iris_fig}
            \includegraphics[width=0.3\textwidth]{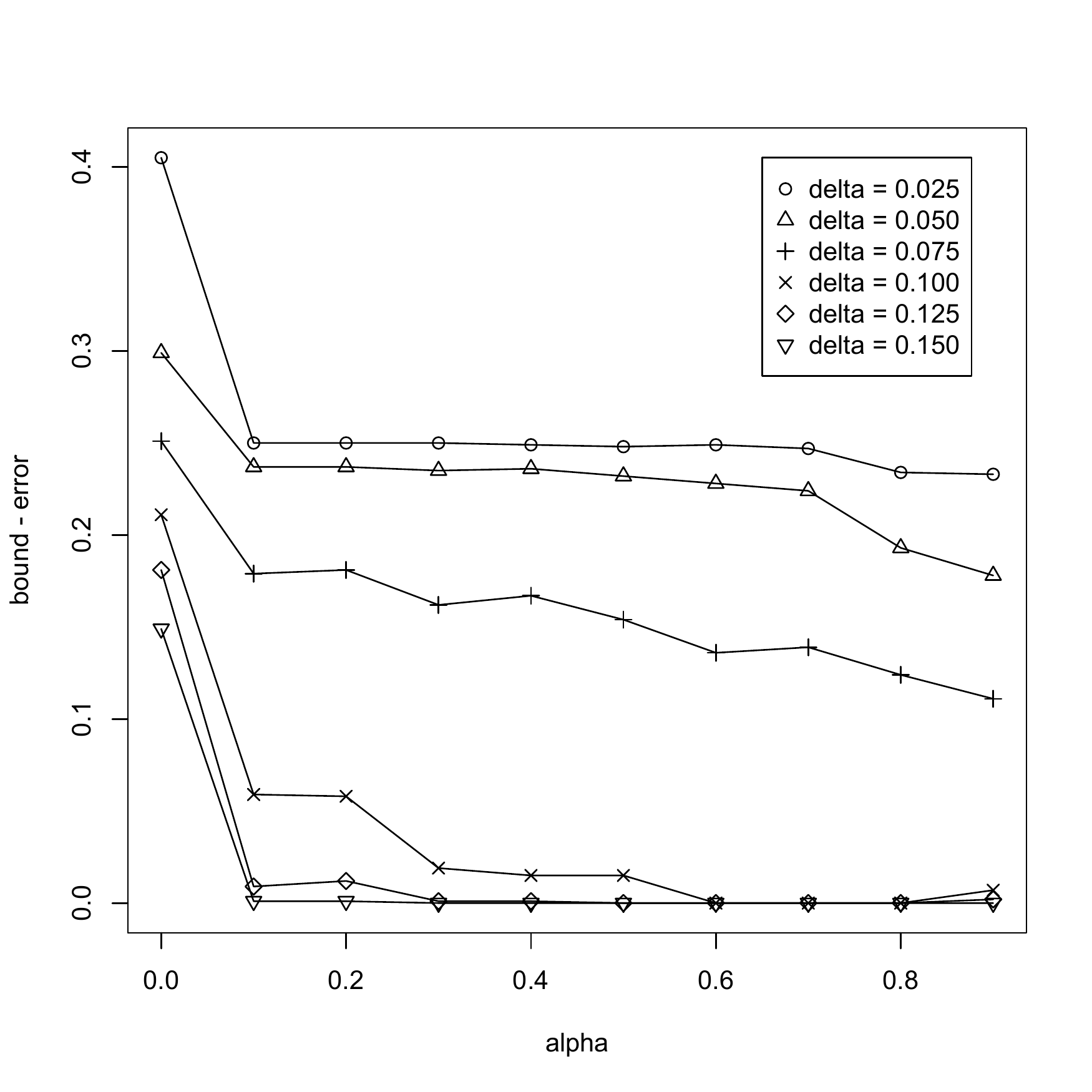}
        }
        \subfigure[Linear Data]{
           \label{vary_linear_fig}
           \includegraphics[width=0.3\textwidth]{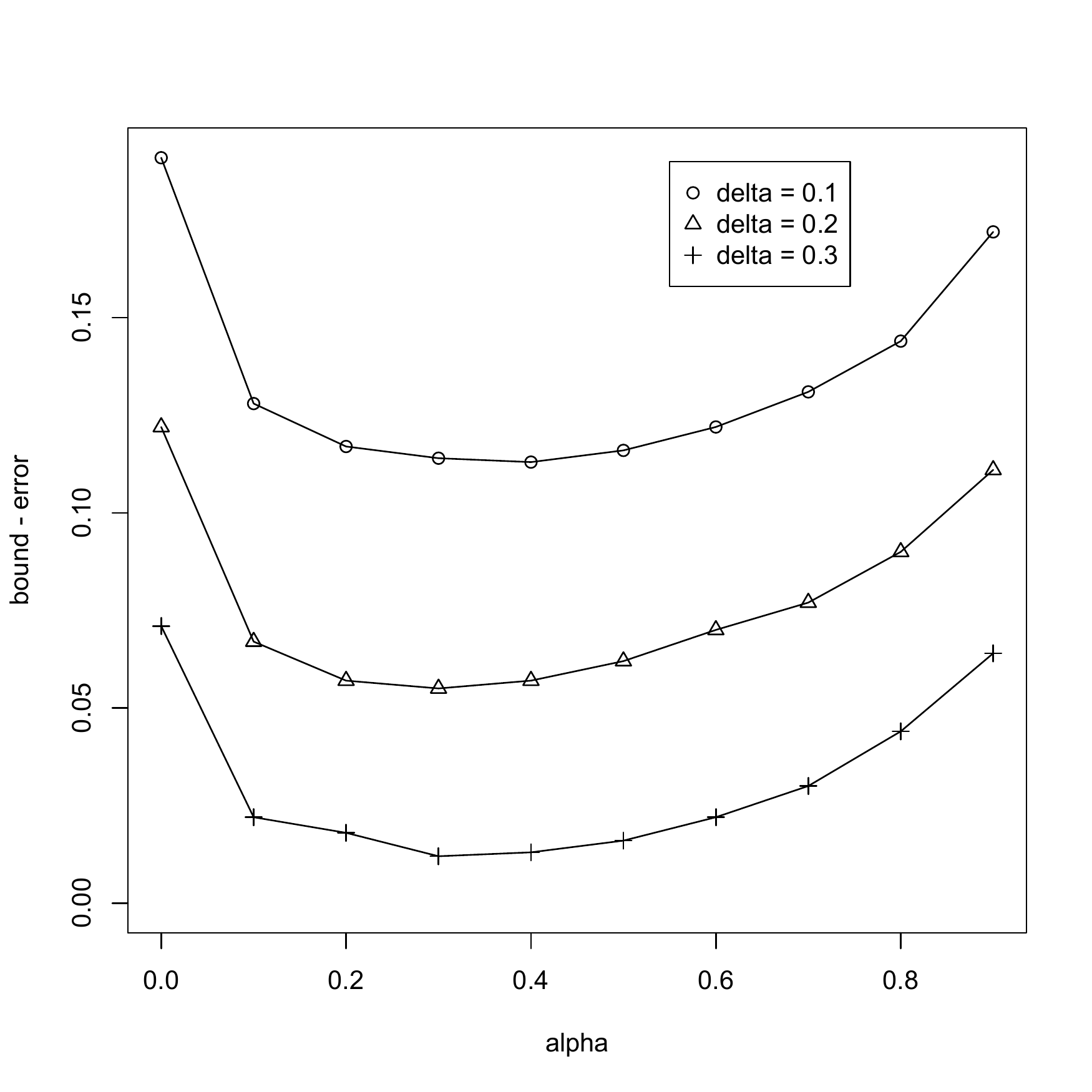}
        }\\ 
        \subfigure[Nonlinear Data]{
            \label{vary_nonlinear_fig}
            \includegraphics[width=0.3\textwidth]{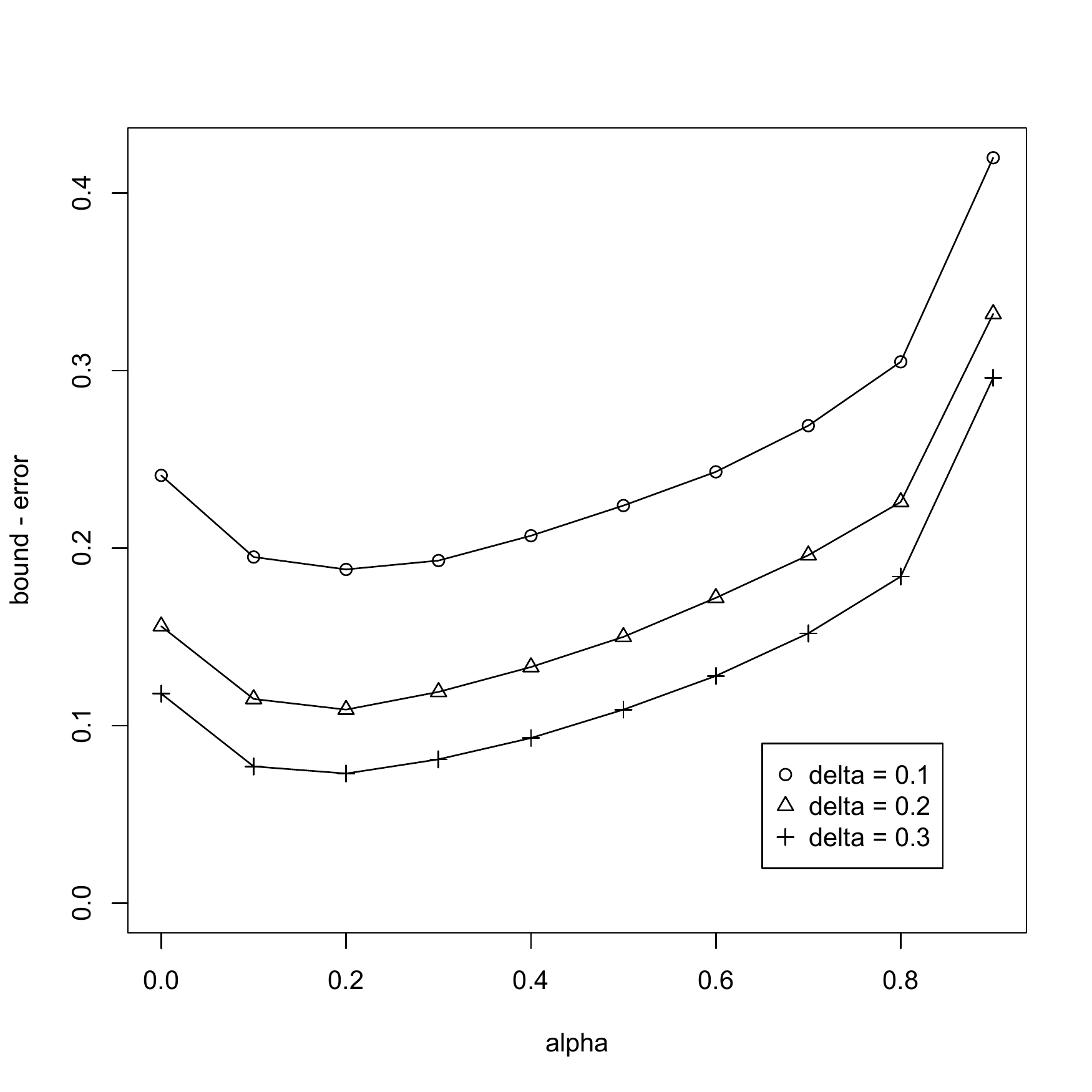}
        }
        \subfigure[Pima Indian Data]{
            \label{vary_pima_fig}
            \includegraphics[width=0.3\textwidth]{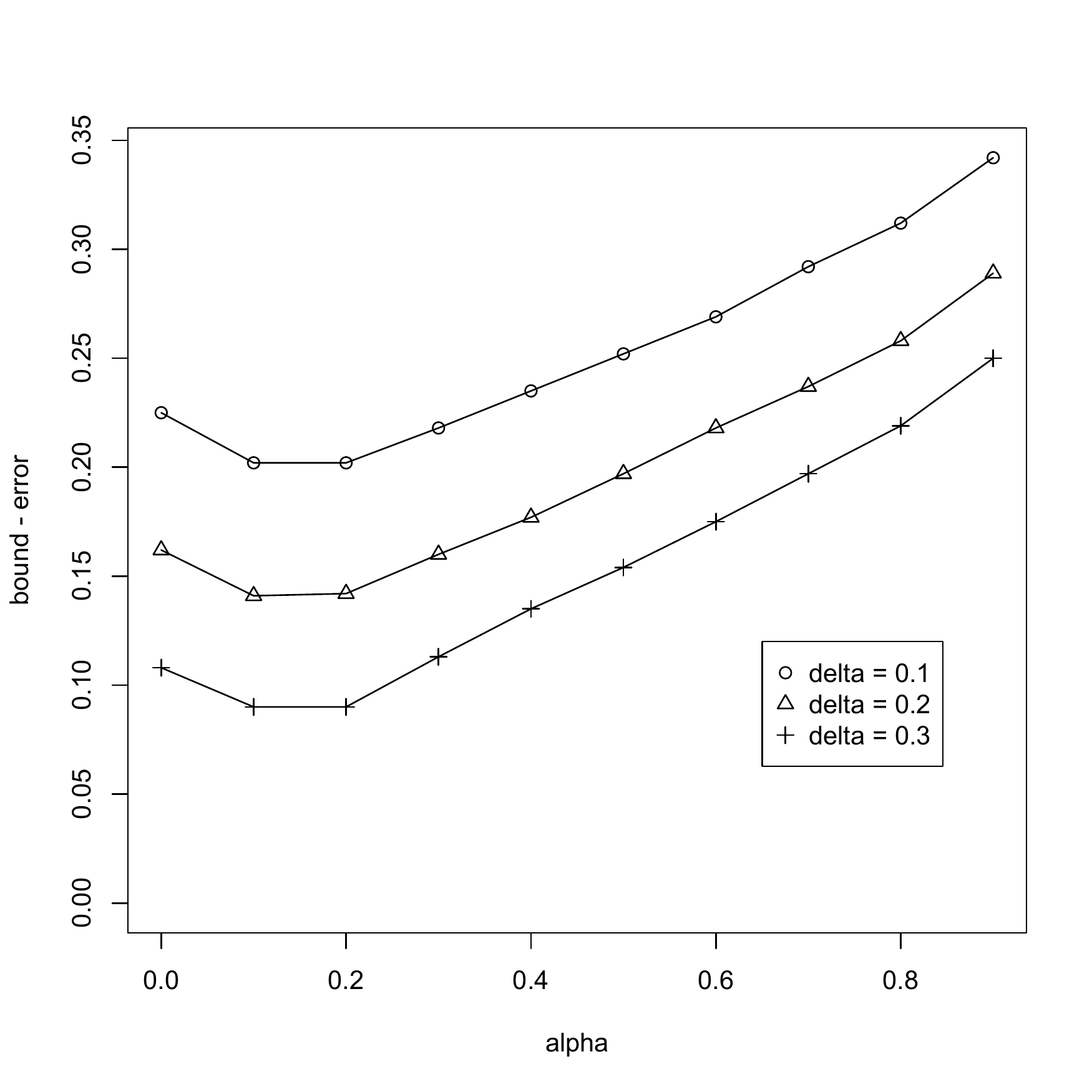}
        }

    \end{center}
    \caption{
    Bound effectiveness as a function of $\alpha$.
     }
   \label{vary_fig}
\end{figure}

\subsection{Iris Data}


\begin{table} 
\caption{Iris Data -- Comparing Scoring Functions} 
\label{iris_table} 
\centering 
\begin{tabular}{|p{1.0in}|p{1.0in}|p{1.0in}|p{1.0in}|} \hline 
\multicolumn{1}{|c|}{$\delta$} & 
\multicolumn{1}{|c|}{Error} & 
\multicolumn{1}{|c}{NNSF - Error} & 
\multicolumn{1}{|c|}{ESF - Error} \\ \hline 

\multicolumn{1}{|c|}{0.025} & 
\multicolumn{1}{|c|}{0.000$\pm$0.000} & 
\multicolumn{1}{|c|}{0.249$\pm$0.016} & 
\multicolumn{1}{|c|}{0.408$\pm$0.121} \\ \hline 

\multicolumn{1}{|c|}{0.050} & 
\multicolumn{1}{|c|}{0.000$\pm$0.000} & 
\multicolumn{1}{|c|}{0.235$\pm$0.060} & 
\multicolumn{1}{|c|}{0.296$\pm$0.104} \\ \hline 

\multicolumn{1}{|c|}{0.075} & 
\multicolumn{1}{|c|}{0.000$\pm$0.000} & 
\multicolumn{1}{|c|}{0.154$\pm$0.122} & 
\multicolumn{1}{|c|}{0.247$\pm$0.084} \\ \hline 

\multicolumn{1}{|c|}{0.100} & 
\multicolumn{1}{|c|}{0.000$\pm$0.000} & 
\multicolumn{1}{|c|}{0.017$\pm$0.063} & 
\multicolumn{1}{|c|}{0.213$\pm$0.096} \\ \hline 

\multicolumn{1}{|c|}{0.125} & 
\multicolumn{1}{|c|}{0.000$\pm$0.000} & 
\multicolumn{1}{|c|}{0.000$\pm$0.008} & 
\multicolumn{1}{|c|}{0.182$\pm$0.113} \\ \hline 

\multicolumn{1}{|c|}{0.150} & 
\multicolumn{1}{|c|}{0.000$\pm$0.000} & 
\multicolumn{1}{|c|}{0.000$\pm$0.000} & 
\multicolumn{1}{|c|}{0.142$\pm$0.124} \\ \hline 

\end{tabular} 

\end{table} 

Figure \ref{vary_iris_fig} and Table~\ref{iris_table} show results for a data set involving iris classification. The data set is from the repository of data sets for machine learning maintained by the University of California at Irvine, which is available online. The data set contains examples for three types of iris; we use only the examples for the first two types in order to produce binary classification problems. This leaves 100 examples, with 50 from each class. Each example has four input dimensions. We use \textit{t} = 40 training examples and \textit{w} = 4 working examples for each trial. The classification method is 1-nearest neighbor. The iris data are easy to classify, as indicated by the fact that the errors are always zero. 

For both scoring functions, the bounds use as $Q$ a set of permutations that makes each size 4 subset of the 44 examples the last 4 examples in the sequence exactly once. (This is equivalent, except for random tie-breaking, to using all 44! permutations of the sequence of examples as $Q$.) Bax and Callejas call this a \textit{complete filter} \cite{baxcallejas}.

Figure \ref{vary_iris_fig} shows that, for most values of $\delta$, the bound improves substantially from $\alpha = 0.0$ to $\alpha = 0.1$, which goes from considering only agreement with the nearest neighbor among training examples to evaluate the label of each working example to giving other nearby neighbors a role as well. Then the bounds stay about the same as $\alpha$ increases. 

Table~\ref{iris_table} shows that NNSF with $\alpha = 0.5$ has a statistically significant advantage over ESF over a full range of $\delta$ values. For $\delta=0.1$ and greater, NNSF produces bounds that are at least an order of magnitude tighter than those produced by ESF. For $\delta=0.15$, NNSF returned a bound that is equal to the actual error rate of zero for all 1000 problems. 

NNSF performs so well because the classes in the iris data are well-separated. The nearest several neighbors to an example are usually all from the same class as the example. As a result, NNSF easily rejects as unlikely those assignments that mislabel working examples, because their several closest neighbors among the training examples all disagree with the incorrect label.  

\subsection{Data with a Linear Class Boundary}


\begin{table} 
\caption{Linear Class Boundaries -- Comparing Scoring Functions} 
\label{linear_table} 
\centering 
\begin{tabular}{|p{1.0in}|p{1.0in}|p{1.0in}|p{1.0in}|} \hline 
\multicolumn{1}{|c|}{$\delta$} & 
\multicolumn{1}{|c|}{Error} & 
\multicolumn{1}{|c}{NNSF - Error} & 
\multicolumn{1}{|c|}{ESF - Error} \\ \hline 

\multicolumn{1}{|c|}{0.100} & 
\multicolumn{1}{|c|}{0.067$\pm$0.078} & 
\multicolumn{1}{|c|}{0.114$\pm$0.097} & 
\multicolumn{1}{|c|}{0.194$\pm$0.100} \\ \hline 

\multicolumn{1}{|c|}{0.200} & 
\multicolumn{1}{|c|}{0.073$\pm$0.083} & 
\multicolumn{1}{|c|}{0.051$\pm$0.109} & 
\multicolumn{1}{|c|}{0.118$\pm$0.101} \\ \hline 

\multicolumn{1}{|c|}{0.300} & 
\multicolumn{1}{|c|}{0.068$\pm$0.079} & 
\multicolumn{1}{|c|}{0.010$\pm$0.106} & 
\multicolumn{1}{|c|}{0.066$\pm$0.097} \\ \hline 

\end{tabular} 

\end{table} 

Figure \ref{vary_linear_fig} and Table~\ref{linear_table} show results for randomly generated data. The data consist of 1100 examples drawn uniformly at random from a three-dimensional input cube with length one on each side. The class label is zero if the input is from the left half of the cube and one if the input is from the right half of the cube. For these tests, there are \textit{t} = 100 training examples and \textit{w} = 10 working examples, using 1-nearest neighbor classification. 

For both scoring functions, the bounds use as $Q$ a set of 1000 permutations drawn uniformly at random without replacement from a set of permutations that makes each size 10 subset of the 110 examples the last 10 examples exactly once. Bax and Callejas call this a \textit{sample filter} \cite{baxcallejas}. Sampling is used to reduce computation. 

Figure \ref{vary_linear_fig} shows that, as for the iris data, the largest improvement in the bounds for data with a linear class boundary comes from changing $\alpha$ from $0.0$ to $0.1$, to give multiple neighbors a role in evaluating assignments. From there, increasing $\alpha$ produces stronger bounds until $\alpha = 0.4$ or $\alpha = 0.5$. Then, increasing $\alpha$ more produces weaker bounds, as more distant neighbors that are less likely to be from the same class as the working example being evaluated play a stronger role in the scoring function.

Table~\ref{linear_table} shows that NNSF with $\alpha = 0.5$ produces statistically significantly better bounds than ESF for this data. The ratio of the difference between bound and error for NNSF to that of ESF increases as bound certainty parameter $\delta$ increases. NNSF performs well for this data because, as for the iris data, the nearest neighbors to each working example among the training examples are usually from the same class as the working example. As a result, NNSF rejects as unlikely most assignments that mislabel working examples.

\subsection{Data with a Nonlinear Class Boundary}


\begin{table} 
\caption{Nonlinear Class Boundaries -- Comparing Scoring Functions} 
\label{nonlinear_table} 
\centering 
\begin{tabular}{|p{1.0in}|p{1.0in}|p{1.0in}|p{1.0in}|} \hline 
\multicolumn{1}{|c|}{$\delta$} & 
\multicolumn{1}{|c|}{Error} & 
\multicolumn{1}{|c}{NNSF - Error} & 
\multicolumn{1}{|c|}{ESF - Error} \\ \hline 

\multicolumn{1}{|c|}{0.100} & 
\multicolumn{1}{|c|}{0.186$\pm$0.129} & 
\multicolumn{1}{|c|}{0.242$\pm$0.141} & 
\multicolumn{1}{|c|}{0.245$\pm$0.145} \\ \hline 

\multicolumn{1}{|c|}{0.200} & 
\multicolumn{1}{|c|}{0.180$\pm$0.125} & 
\multicolumn{1}{|c|}{0.155$\pm$0.151} & 
\multicolumn{1}{|c|}{0.160$\pm$0.143} \\ \hline 

\multicolumn{1}{|c|}{0.300} & 
\multicolumn{1}{|c|}{0.177$\pm$0.122} & 
\multicolumn{1}{|c|}{0.105$\pm$0.157} & 
\multicolumn{1}{|c|}{0.116$\pm$0.138} \\ \hline 

\end{tabular} 

\end{table}

Figure \ref{vary_nonlinear_fig} and Table~\ref{nonlinear_table} show results for randomly generated data with a nonlinear class boundary. The data have the same characteristics as in the previous test, except that each class label is determined by the XOR of whether the input is in the left half of the cube, the bottom half of the cube, and the front half of the cube. In other words, the cube is cut into eight sub-cubes, and each sub-cube has a different class than the three sub-cubes with which it shares a side. This class scheme introduces more error than for the data with a linear class boundary. Similar to the data with linear class boundaries, 1-nearest neighbor classification is used, and the bound method uses a random sample of 1000 permutations as $Q$.

Figure \ref{vary_nonlinear_fig} shows results similar to those for data with a linear class boundary, except that increasing $\alpha$ beyond $0.2$ produces weaker bounds. This occurs for lower $\alpha$ with this data set, because working examples in this data set are more likely to have some near neighbors from a different class than working examples in the data set with a linear class boundary.

For the data with nonlinear class boundaries, Table~\ref{nonlinear_table} shows that NNSF with $\alpha = 0.5$ and ESF perform similarly well. For $\delta = 0.1$ and $0.2$, the differences in bounds produced by NNSF and by ESF are not statistically significant. For $\delta = 0.3$, the difference is statistically significant, but small. For this data, working examples are likely to have near neighbors among the training examples that are from a different class. As a result, NNSF does not have a strong advantage over ESF. 

\subsection{Pima Indian Diabetes Data}


\begin{table} 
\caption{Pima Indian Diabetes Data -- Comparing Scoring Functions} 
\label{pima_table} 

\centering 
\begin{tabular}{|p{1.0in}|p{1.0in}|p{1.0in}|p{1.0in}|} \hline 
\multicolumn{1}{|c|}{$\delta$} & 
\multicolumn{1}{|c|}{Error} & 
\multicolumn{1}{|c}{NNSF - Error} & 
\multicolumn{1}{|c|}{ESF - Error} \\ \hline 

\multicolumn{1}{|c|}{0.100} & 
\multicolumn{1}{|c|}{0.322$\pm$0.136} & 
\multicolumn{1}{|c|}{0.252$\pm$0.139} & 
\multicolumn{1}{|c|}{0.225$\pm$0.147} \\ \hline 

\multicolumn{1}{|c|}{0.200} & 
\multicolumn{1}{|c|}{0.309$\pm$0.135} & 
\multicolumn{1}{|c|}{0.197$\pm$0.140} & 
\multicolumn{1}{|c|}{0.162$\pm$0.146} \\ \hline 

\multicolumn{1}{|c|}{0.300} & 
\multicolumn{1}{|c|}{0.314$\pm$0.139} & 
\multicolumn{1}{|c|}{0.154$\pm$0.150} & 
\multicolumn{1}{|c|}{0.108$\pm$0.151} \\ \hline 

\end{tabular} 
\end{table}

Figure \ref{vary_pima_fig} and Table~\ref{pima_table} show results for data related to diabetes among Pima Indians. The data set is available from the online repository maintained by the University of California at Irvine. The data set has 768 examples: 500 from one class and 268 from another. Each example has eight input dimensions. Since the input dimensions have different scales, we normalize the data, translating and scaling each input dimension to give it  mean zero and standard deviation one. We use \textit{t} = 200 training examples and \textit{w} = 12 working examples for each trial, with $Q$ a random sample of 100 permutations. The tests use 1-nearest neighbor classification. 

The results in Figure \ref{vary_pima_fig} are similar to those for random data with a nonlinear class boundary. Like that data, working examples in the Pima Indian data are likely to have near neighbors from a different class. Still, as for the nonlinear class boundary data, there is some improvement from using small positive values of $\alpha$ rather than $\alpha=0.0$. These data are difficult to classify, as shown by the high error rates in Table~\ref{pima_table}. For this data, ESF outperforms NNSF with $\alpha = 0.5$. However, smaller values of $\alpha$ make NNSF an improvement over ESF even for this data set, as shown in Figure~\ref{vary_pima_fig}.


 \section{Discussion}
This paper developed a new scoring function for worst likely assignment error bounds. For each of our example data sets, the scoring function improved error bounds for some setting of the parameter $\alpha$, and it improved the error bounds for a blind choice of $\alpha$ for the data sets that produced accurate classifiers. For the accurate classifiers, the method produced effective error bounds even with 100 or fewer training examples. 
 
 One challenge for the future is to develop faster methods to compute worst likely assignment error bounds. We can speed up the permutation test for each assignment by sampling permutations, as we did in the tests. It may be possible to use fewer permutations by a cleverer method of sampling (we used uniform sampling) or by fitting the sampled results to a distribution, as in \cite{knijnenburg09}. 
 
 The greater challenge is to avoid explicitly running a permutation test for each assignment, since the number of assignments is exponential in the number of working examples. For 1-nearest neighbor classifiers, there is a dynamic programming method that computes a worst likely assignment error bound without explicit computation for each assignment \cite{baxcallejas}. The method produces an error bound in time polynomial in the number of in-sample examples. It would be very useful to extend that method to other types of classifiers, and selecting the right scoring functions may be the key. (An alternative method to achieve polynomial-time computation is to partition the working set, produce an error bound for each partition, and use a union bound over the partitions \cite{li12}. But this approach produces looser bounds than validation over the whole working set at once.)
 
 Another challenge is to improve worst likely assignment error bounds for network classifiers \cite{sen2008,getoor07,kolacyzk10,london12}. Since many networks grow by accretion, for example friends inviting friends to join social networks, the nodes are not necessarily drawn i.i.d. This makes developing permutation tests for the nodes challenging. However, some subsets of nodes may be drawn i.i.d. \cite{bax13}, allowing permutation tests within those subsets. There is some work using permutation tests for hypothesis testing in social networks \cite{belo13,belo12,lafond10,anagnostopoulos08} and some work on applying worst likely assignments to produce error bounds for them \cite{li12}. The challenge for the future is to develop scoring functions specifically designed for network classifiers in order to improve those bounds. 
 
 Finally, it would be interesting to explore the concept of best error bounds. Just as there are established criteria for best estimators in statistics (see e.g. \cite{hoel54} pp. 200--202), there should be reasonable criteria for best error bounds. These criteria may drive the discovery and development of new scoring functions.

\newpage

\bibliographystyle{unsrt}
\bibliography{bax}

\end{document}